\documentclass[runningheads]{llncs}
\usepackage{graphicx}
\usepackage{amsmath, amssymb}
\usepackage{mathtools}
\usepackage{thmtools, thm-restate}
\usepackage{bbm}
\usepackage{hyperref}
\usepackage{tikz}
\usepackage{pgfplots}\pgfplotsset{compat=1.14}
\usepackage[utf8]{inputenc}
\usepackage{3dplot}
\usetikzlibrary{positioning}
\usetikzlibrary{angles}
\usepackage{subcaption}
\usepackage{float}
\usepackage{gensymb}
\usepackage{bm}

\DeclareMathOperator{\bias}{Bias}

\newcommand*\dif{\mathop{}\!\mathrm{d}}

\pgfplotsset{
colormap={cool}{rgb255(0cm)=(205, 230, 249); rgb255(1cm)=(0,128,255); rgb255(2cm)=(255,0,255)}
}
\usepackage[symbol]{footmisc}

\newcommand{\daggerthanks}[1]{
\let\oldthefootnote=\thefootnote
\setcounter{footnote}{1}
\renewcommand{\thefootnote}{\fnsymbol{footnote}}
\thanks{#1}
\let\thefootnote=\oldthefootnote
}

\makeatletter
\renewcommand*{\@fnsymbol}[1]{\ensuremath{\ifcase#1\or \star\or \dagger\or \ddagger\or
    \mathsection\or \mathparagraph\or \|\or \star\star\or \dagger\dagger
    \or \ddagger\ddagger \else\@ctrerr\fi}}
\makeatother

\begin{document}
\title{The Futility of Bias-Free Learning and Search\thanks{Supported, in part, by a grant from the Walter Bradley Center for Natural and Artificial Intelligence.}}

%
%

\author{George D.\ Monta\~nez\orcidID{0000-0002-1333-4611} \and Jonathan Hayase\orcidID{0000-0002-3757-6586}\thanks{denotes equal contribution.} \and
Julius Lauw\orcidID{0000-0003-4201-0664}$^{\dagger}$ \and Dominique Macias\orcidID{0000-0002-6506-4094}$^{\dagger}$ \and Akshay Trikha\orcidID{0000-0001-8207-6399}$^{\dagger}$ \and Julia Vendemiatti\orcidID{0000-0002-6547-9601}$^{\dagger}$}

\authorrunning{G. Monta\~nez et al.}
%
\institute{AMISTAD Lab, Harvey Mudd College, Claremont CA 91711, USA\\
\email{\{gmontanez,jhayase,julauw,dmacias,atrikha,jvendemiatti\}@hmc.edu}}
\maketitle              
\begin{abstract}
Building on the view of machine learning as search, we demonstrate the necessity of bias in learning, quantifying the role of bias (measured relative to a collection of possible datasets, or more generally, information resources) in increasing the probability of success. For a given degree of bias towards a fixed target, we show that the proportion of favorable information resources is strictly bounded from above. Furthermore, we demonstrate that bias is a conserved quantity, such that no algorithm can be favorably biased towards many distinct targets simultaneously. Thus bias encodes trade-offs. The probability of success for a task can also be measured geometrically, as the angle of agreement between what holds for the actual task and what is assumed by the algorithm, represented in its bias. Lastly, finding a favorably biasing distribution over a fixed set of information resources is provably difficult, unless the set of resources itself is already favorable with respect to the given task and algorithm.

\end{abstract}

\section{Introduction}

Imagine you are on a routine grocery shopping trip and plan to buy some bananas. You know that the store carries both good and bad bananas which you must search through. There are multiple ways you can go about your search. One way is to randomly pick any ten bananas available on the shelf, which can be regarded as a form of unbiased search. Alternatively, you could introduce some bias to your search by only picking those bananas that are neither underripe nor overripe. Based on your past experiences from eating bananas, there is a better chance that these bananas will taste better. The proportion of good bananas retrieved in your biased search is greater than the same proportion in an unbiased search; you used your prior knowledge about tasty bananas. This common routine shows how bias enables us to conduct more successful searches based on prior knowledge of the search target.

Viewing these decision-making processes through the lens of machine learning, we analyze how algorithms tackle learning problems under the influence of bias. Will we be better off without the existence of bias in machine learning algorithms? Our goal in this paper is to formally characterize the direct relationship between the performance of machine learning algorithms and their underlying biases. Without bias, machine learning algorithms will not perform better than uniform random sampling, on average. Yet to the extent an algorithm is biased toward some target is the extent to which it is biased against all remaining targets. As a consequence, no algorithm can be biased towards all targets. Therefore, bias represents the trade-offs an algorithm makes in how to respond to data. 

We approach this problem by analyzing the performance of search algorithms within the algorithmic search framework introduced by Monta\~nez~\cite{montanez2017fof}. This framework applies to common machine learning tasks such as classification, regression, clustering, optimization, reinforcement learning, and the general machine learning problems considered in Vapnik's learning framework \cite{montanez2017dissertation}. We derive results characterizing the role of bias in successful search, extending Famine of Forte results~\cite{montanez2017fof} for a fixed search target and varying information resources. Our results for bias-free search then directly apply to bias-free learning, showing the extent to which bias is necessary for successful learning and quantifying how difficult it is to find a distribution with favorable bias for a particular target.

\section{Related Work}

 Schaffer's seminal work \cite{schaffer1994conservation} showed that generalization performance for classification problems is a conserved quantity, such that favorable performance on a particular subset of problems will always be offset and balanced by poor performance over the remaining problems. Similarly, we show that bias is also a conserved quantity for any set of information resources. While Schaffer studied the performance of a single algorithm over different learning classes, Wolpert and Macready's ``No Free Lunch Theorems for Optimization" \cite{wolpert1997nfl} established that all optimization algorithms have the same performance when uniformly averaged over all possible cost functions. They also provided a geometric intuition for this result by defining an inner product which measures the alignment between an algorithm and a given prior over problems. This shows that no algorithm can be simultaneously aligned with all possible priors. In the context of the search framework, we define the geometric divergence as a measure of alignment between a search algorithm and a target in order to bound the proportion of favorable search problems. 
 
 While No Free Lunch Theorems are widely recognized as landmark ideas in machine learning, McDermott claims that No Free Lunch results are often misinterpreted and are practically insignificant for many real-world problems~\cite{McDermott2019}. This is because algorithms are commonly tailored to a specific subset of problems in the real world, but No Free Lunch requires that we consider the set of all problems that are closed under permutation. These arguments towards the impracticality of No Free Lunch results are less relevant to our work here, since we evaluate the proportion of successful problems instead of considering the mean performance over the set of all problems. As such, our results are also applicable to sets of problems that are not closed under permutation, as a generalization of No Free Lunch results.

 In ``The Famine of Forte: Few Search Problems Greatly Favor Your Algorithm", Monta\~nez~\cite{montanez2017fof} reduces machine learning problems to search problems and develops a rigorous search framework to generalize No Free Lunch ideas. He strictly bounds the proportion of problems that are favorable for a fixed algorithm and shows that no single algorithm can perform well over a large fraction of search problems. Extending these results to fixed search targets, we show that there are also strict bounds on the proportion of favorable information resources, and that the bound relaxes with the introduction of bias.

Our notion of bias developed here relates to ideas introduced by Mitchell~\cite{needforbiases}. According to Mitchell, a completely unbiased classification algorithm cannot generalize beyond training data. He argued that the ability of a learning algorithm to generalize depends on incorporating biases, which means making assumptions beyond strict consistency with training data. These biases may include prior knowledge of the domain, preferences for simplicity, and awareness of the algorithm's real-world application. We strengthen Mitchell's argument with a mathematical justification for the need for bias in improving learning performance.

G\"{u}l\c{c}ehre and Bengio empirically support Mitchell's ideas by investigating the nature of training barriers affecting the generalization performance of black-box machine learning algorithms~\cite{priorinfo}. Using the Structured Multi-Layer Perceptron (SMLP) neural network architecture, they showed that pre-training the SMLP with hints based on prior knowledge of the task generalizes more efficiently as compared to an SMLP pre-trained with random initializers. Furthermore, Ulyanov et al.\ explore the success of deep convolutional networks applied to image generation and restoration~\cite{deepimagepriors}. By applying untrained convolutional networks to image reconstruction with competitive success to trained ones, they show that the impressive performance of these networks is not due to learning alone. They highlight the importance of inductive bias, which is built into the structure of these generator networks, in achieving this high level of success. In a similar vein, Runarsson and Yao establish that bias is an essential component in constrained evolutionary optimization search problems~\cite{searchbias}. It is experimentally shown that carefully selecting an appropriate constraint handling method and applying a biasing penalty function enhances the probability of locating feasible solutions for evolutionary algorithms. Inspired by the results obtained from these experimental studies, we formulate a theoretical validation of the role of bias in generalization performance for learning problems.

\section{The Search Framework}
    \subsection{The Search Problem}
    We formulate machine learning problems as search problems using the algorithmic search framework \cite{montanez2017fof}. Within the framework, a search problem is represented as a 3-tuple $(\mathrm{\Omega}, T, F)$. The finite search space from which we can sample is $\mathrm{\Omega}$. The subset of elements in the search space that we are searching for is the target set $T$. A target function that represents $T$ is an $|\mathrm{\Omega}|$-length vector with entries having value 1 when the corresponding elements of $\mathrm{\Omega}$ are in the target set and 0 otherwise. The external information resource $F$ is a binary string that provides initialization information for the search and evaluates points in $\mathrm{\Omega}$, acting as an oracle that guides the search process. 
    
    \subsection{The Search Algorithm}
    Given a search problem, a history of elements already examined, and information resource evaluations, an algorithmic search is a process that decides how to query elements of $\mathrm{\Omega}$. As the search algorithm samples, it adds the record of points queried and information resource evaluations, indexed by time, to the search history. If the algorithm queries an element $\omega \in T$ at least once during the course of its search, we say that the search is successful. Figure \ref{fig:jellyfish} visualizes the search algorithm. 
    
    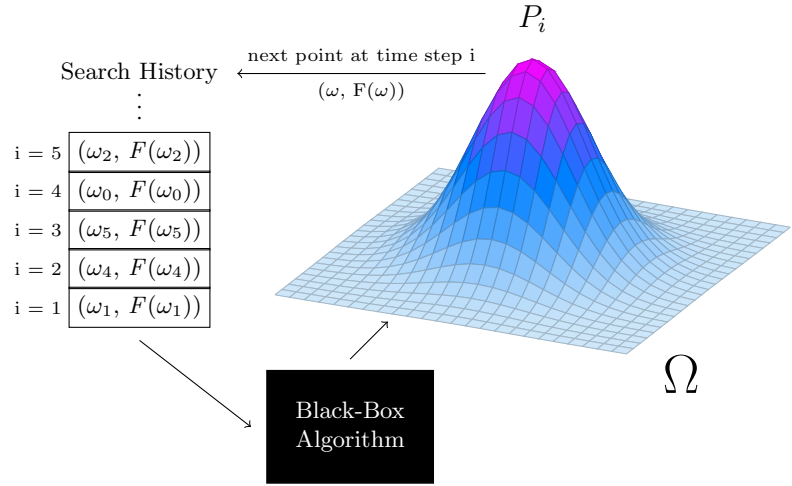
\begin{figure}
        \centering
\def\centerx{2}
\def\centery{-1}
\begin{tikzpicture}
    \begin{axis}[hide axis]
    \addplot3[surf, domain=-2:6,domain y=-5:3] 
        {exp(-( (x-\centerx)^2 + (y-\centery)^2)/3 )};
    \node[text centered] at (axis cs:\centerx, \centery, 1.20) {\large $P_i$};
    \end{axis}
    \node at (5.4,0) {\huge{$\mathrm{\Omega}$}};
    \draw[->] (\centerx + 0.8,4) -- node[above, text centered] {\scriptsize next point at time step i} node[below, text centered] {\scriptsize ($\omega$, F($\omega$))} (-0.5,4);
    \node[draw, fill=black!, text=white, text width=2cm, minimum height=1.5cm, text centered] at (1,-0.7) {Black-Box Algorithm};
    \draw[->] (1, 0.2) -- (1.5, 0.7);
    \node[minimum width=1.65cm, text centered] at (-1.8,4.01) 
        {\small Search History}; 
     \foreach \y in {3.4,3.55, 3.7} 
        \node[ minimum width=2cm, minimum height=0.3cm, text centered] at (-1.8,\y) {$\cdot$};
    \node[draw, minimum width=1.65cm, text centered] at (-1.8,2.95) 
        {\footnotesize ($\omega_2$, \textit{F}($\omega_2$))}; 
    \node[text centered] at (-3.15,2.95) {\scriptsize i = 5};
    \node[draw, minimum width=1.65cm, text centered] at (-1.8,2.435) 
        {\footnotesize ($\omega_0$, \textit{F}($\omega_0$))};
    \node[text centered] at (-3.15,2.435) {\scriptsize i = 4};
    \node[draw, minimum width=1.65cm, text centered] at (-1.8,1.92) 
        {\footnotesize ($\omega_5$, \textit{F}($\omega_5$))};
    \node[text centered] at (-3.15,1.92) {\scriptsize i = 3};
    \node[draw, minimum width=1.65cm, text centered] at (-1.8,1.405) 
        {\footnotesize ($\omega_4$, \textit{F}($\omega_4$))};
    \node[text centered] at (-3.15,1.405) {\scriptsize i = 2};    
    \node[draw, minimum width=1.65cm, text centered] at (-1.8,0.89) 
        {\footnotesize ($\omega_1$, \textit{F}($\omega_1$))};
    \node[text centered] at (-3.15,0.89) {\scriptsize i = 1};
    \draw[->] (-1.8, 0.45) -- (-0.32, -0.7);
\end{tikzpicture}
        \caption{As a black-box optimization algorithm samples from $\mathrm{\Omega}$, it produces an  associated probability distribution $P_i$ based on the search history. When a sample $\omega_k$ corresponding to location $k$ in $\mathrm{\Omega}$ is evaluated using the external information resource $F$, the tuple ($\omega_k$, $F(\omega_k)$) is added to the search history.}
        \label{fig:jellyfish}
    \end{figure}
    
    \subsection{Measuring Performance}
    Within this search framework, we measure a learning algorithm's performance by examining the expected per-query probability of success. This measure is more effective than measuring an algorithm's total probability of success, since the number of sampling steps may vary depending on the algorithm used. Furthermore, the per query probability of success naturally accounts for sampling procedures that may involve repeatedly sampling the same points in the search space, as is the case of genetic algorithms \cite{goldberg1999genetic,reeves2002genetic}. Thus, this measure effectively handles search algorithms that balance exploration and exploitation. 
    
    The expected per-query probability of success is defined as
    \[ q(T,F) = \mathbb{E}_{\tilde{P}, H} \Bigg[ \frac{1}{|\tilde{P}|} \sum_{i=1}^{|\tilde{P}|} P_i(\omega \in T) \Bigg| F \Bigg] \]
    where $\tilde{P}$ is a sequence of probability distributions over the search space (where each timestep \(i\) produces a distribution $P_i$), \(T\) is the target, \(F\) is the information resource, and \(H\) is the search history. The number of queries during a search is equal to the length of the probability distribution sequence,  $|\tilde{P}|$.
    
\section{Main Results}
    We present and explain our main results in this section. Note that full proofs for the following results can be found in the Appendix. We proceed by defining our measures of bias and target divergence, then show conservation results of bias and give bounds on the probability of successful search and the proportion of favorable search problems given a fixed target. 
    \begin{definition}
        \label{def:bias_D}
        (Bias between a distribution over information resources and a fixed target) Let $\mathcal{D}$ be a distribution over a  space of information resources $\mathcal{F}$ and let $F \sim \mathcal{D}$. For a given  $\mathcal{D}$ and a fixed $k$-hot target function \(\bm{t}\),
        \begin{align*}
            \bias(\mathcal{D}, \bm{t}) 
            &= \mathbb{E}_{\mathcal{D}} \left[\bm{t}^\top \overline{P}_{F}\right] - \frac{k}{|\Omega|} \\
            &= \bm{t}^\top \mathbb{E}_{\mathcal{D}}\left[\,\overline{P}_{F}\right] -    \frac{\|\bm{t}\|^2}{|\Omega|} \\
            &= \bm{t}^\top  \int_{\mathcal{F}} \overline{P}_{f} \mathcal{D}(f) \dif f - \frac{\|\bm{t}\|^2}{|\Omega|}
        \end{align*}
        where $\overline{P}_{f}$ is the vector representation of the averaged probability distribution (conditioned on $f$) induced on $\Omega$ during the course of the search, which can be shown to imply $q(t,f) = \bm{t}^\top \overline{P}_{f}$.
    \end{definition}

    \begin{definition}
        \label{def:bias_B}
        (Bias between a finite set of information resources and a fixed target) Let $\mathcal{U}[\mathcal{B}]$ denote a uniform distribution over a finite set of information resources \(\mathcal{B}\). For a random quantity $F \sim \mathcal{U}[\mathcal{B}]$, the averaged \(|\Omega|\)-length simplex vector $\overline{P}_{F}$, and a fixed $k$-hot target function \(\bm{t}\), 
        \begin{align*}
            \bias(\mathcal{B}, \bm{t}) 
            &= \mathbb{E}_{\mathcal{U}[\mathcal{B}]}[\bm{t}^\top \overline{P}_{F}] - \frac{k}{|\Omega|} \\
            &= \bm{t}^\top \mathbb{E}_{\mathcal{U}[\mathcal{B}]}[\overline{P}_{F}] - \frac{k}{|\Omega|} \\
            &= \bm{t}^\top \left( \frac{1}{|\mathcal{B}|}\sum_{f \in \mathcal{B}} \overline{P}_{f} \right) - \frac{\|\bm{t}\|^{2}}{|\Omega|}.
        \end{align*}
    \end{definition}
    We define bias as the difference between average performance of a search algorithm on a fixed target over a set of information resources and the baseline search performance for the case of uniform random sampling. Definition \ref{def:bias_D} is a generalized form of Definition \ref{def:bias_B}, characterizing the alignment between a target function and a distribution over information resources instead of a fixed set.
    \begin{definition}
        (Target Divergence) The measure of similarity between a fixed target function \textbf{t} and the expected value of the averaged \(|\Omega|\)-length simplex vector $\overline{P}_{F}$, where $F\sim \mathcal{D}$, is defined as
        \[\theta = \arccos \left ( \frac{\bm{t}^{\top} \mathbb{E}_{\mathcal{D}}[\overline{P}_{F}]}{\|\bm{t}\| \|\mathbb{E}_{\mathcal{D}}[\overline{P}_{F}]\|} \right)\]
    \end{definition}
     Similar to Wolpert and Macready's geometric interpretation of the No Free Lunch theorems in \cite{wolpert1997nfl}, we can evaluate how far a target function $\bm{t}$ deviates from the averaged probability simplex vector $\overline{P}_{f}$ for a given search problem. In this paper, we use cosine similarity to measure the level of similarity between $\bm{t}$ and $\overline{P}_{f}$. Geometrically, the target divergence is the angle between the target vector and the averaged $|\Omega|$-length simplex vector. Figure \ref{fig:tardiv} depicts the target divergence for various levels of alignments between $\bm{t}$ and $\overline{P}_{f}$.
     
\tdplotsetmaincoords{70}{130}
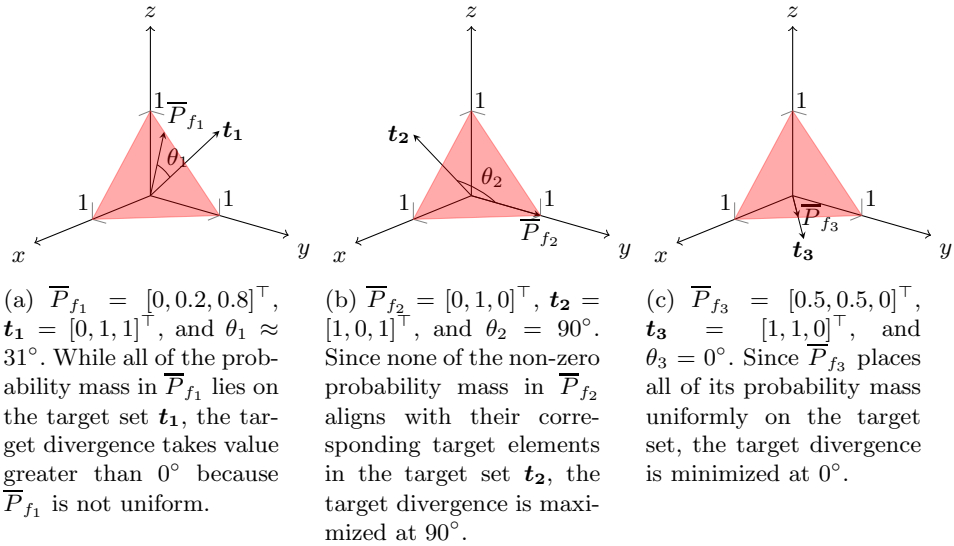
\begin{figure}[t] 
    \centering
    \begin{subfigure}[b]{0.3\textwidth} 
        \begin{tikzpicture}[tdplot_main_coords, scale=1.2]
            \def\laxis{2}
            \def\ltriangle{1}
            \def\ltick{.2}
            
            \coordinate (O) at (0,0,0);
    
            \draw [->] (0,0,0) -- (\laxis,0,0) node [anchor=north east] {$x$};
            \draw [->] (0,0,0) -- (0,\laxis,0) node [anchor=north west] {$y$};
            \draw [->] (0,0,0) -- (0,0,\laxis) node [anchor=south] {$z$};
            \pgfmathtruncatemacro{\nticks}{floor(\laxis)-1}
            \begin{scope}[
            help lines,
            every node/.style={inner sep=1pt,text=black}
            ]
            \foreach \coord in {1,...,\nticks} {
              \draw (\coord,\ltick,0) -- ++(0,-\ltick,0) -- ++(0,0,\ltick)
              node [pos=1,left] {\coord};
              \draw (\ltick,\coord,0) -- ++(-\ltick,0,0) -- ++(0,0,\ltick)
              node [pos=1,right] {\coord};
              \draw (\ltick,0,\coord) -- ++(-\ltick,0,0) -- ++(0,\ltick,0)
              node [at start,above right] {\coord};
            }
            \draw[-stealth,color=black] (0,0,0) -- (0,1,1) node [anchor= west] {$\bm{t_1}$};
            \draw[-stealth,color=black] (0,0,0) -- (0,.2,.8) node [anchor= south west] {$\overline{P}_{f_1}$};
            
            \tdplotsetrotatedcoords{0}{90}{90}
            
            \tdplotdrawarc[tdplot_rotated_coords,color=black]{(0,0,0)}{0.4}{45}{76}{anchor=south west,color=black}{$\theta_1$}
        
            \end{scope}
            \filldraw [opacity=.33,red] (\ltriangle,0,0) -- (0,\ltriangle,0)
            -- (0,0,\ltriangle) -- cycle;
        \end{tikzpicture}
        \caption{$\overline{P}_{f_1} = [0, 0.2,0.8]^{\top}$, $\bm{t_1} = [0,1,1]^{\top}$, and $\theta_1 \approx 31 \degree$. While all of the probability mass in $\overline{P}_{f_1}$ lies on the target set $\bm{t_1}$, the target divergence takes value greater than $0\degree$ because $\overline{P}_{f_1}$ is not uniform. \\}\label{fig:tardiva}
    \end{subfigure} \hfill
    \begin{subfigure}[b]{0.3\textwidth}
        \begin{tikzpicture}[tdplot_main_coords, scale=1.2]
            \def\laxis{2}
            \def\ltriangle{1}
            \def\ltick{.2}
            
            \coordinate (O) at (0,0,0);
    
            \draw [->] (0,0,0) -- (\laxis,0,0) node [anchor=north east] {$x$};
            \draw [->] (0,0,0) -- (0,\laxis,0) node [anchor=north west] {$y$};
            \draw [->] (0,0,0) -- (0,0,\laxis) node [anchor=south] {$z$};
            \pgfmathtruncatemacro{\nticks}{floor(\laxis)-1}
            \begin{scope}[
            help lines,
            every node/.style={inner sep=1pt,text=black}
            ]
            \foreach \coord in {1,...,\nticks} {
              \draw (\coord,\ltick,0) -- ++(0,-\ltick,0) -- ++(0,0,\ltick)
              node [pos=1,left] {\coord};
              \draw (\ltick,\coord,0) -- ++(-\ltick,0,0) -- ++(0,0,\ltick)
              node [pos=1,right] {\coord};
              \draw (\ltick,0,\coord) -- ++(-\ltick,0,0) -- ++(0,\ltick,0)
              node [at start,above right] {\coord};
            }
            \draw[-stealth,color=black] (0,0,0) -- (1,0,1) node [anchor= east] {$\bm{t_2}$};
            \draw[-stealth,color=black] (0,0,0) -- (0,1,0) node [anchor= north] {$\overline{P}_{f_2}$};
            \tdplotsetrotatedcoords{0}{135}{90}
            \tdplotdrawarc[tdplot_rotated_coords,color=black]{(0,0,0)}{0.35}{00}{90}{anchor=south west,color=black}{$\theta_2$}
            
            \end{scope}
            \filldraw [opacity=.33,red] (\ltriangle,0,0) -- (0,\ltriangle,0)
            -- (0,0,\ltriangle) -- cycle;
        \end{tikzpicture}
        \caption{$\overline{P}_{f_2} =[0,1,0]^{\top} $, $\bm{t_2}= [1,0,1]^{\top}$, and $\theta_2 = 90 \degree$. Since none of the non-zero probability mass in $\overline{P}_{f_2}$ aligns with their corresponding target elements in the target set $\bm{t_2}$, the target divergence is maximized at $90 \degree$.}\label{fig:tardivb}
    \end{subfigure} \hfill
    \begin{subfigure}[b]{0.3\textwidth} 
        \begin{tikzpicture}[tdplot_main_coords, scale=1.2]
            \def\laxis{2}
            \def\ltriangle{1}
            \def\ltick{.2}
            
            \coordinate (O) at (0,0,0);
    
            \draw [->] (0,0,0) -- (\laxis,0,0) node [anchor=north east] {$x$};
            \draw [->] (0,0,0) -- (0,\laxis,0) node [anchor=north west] {$y$};
            \draw [->] (0,0,0) -- (0,0,\laxis) node [anchor=south] {$z$};
            \pgfmathtruncatemacro{\nticks}{floor(\laxis)-1}
            \begin{scope}[
            help lines,
            every node/.style={inner sep=1pt,text=black}
            ]
            \foreach \coord in {1,...,\nticks} {
              \draw (\coord,\ltick,0) -- ++(0,-\ltick,0) -- ++(0,0,\ltick)
              node [pos=1,left] {\coord};
              \draw (\ltick,\coord,0) -- ++(-\ltick,0,0) -- ++(0,0,\ltick)
              node [pos=1,right] {\coord};
              \draw (\ltick,0,\coord) -- ++(-\ltick,0,0) -- ++(0,\ltick,0)
              node [at start,above right] {\coord};
            }
            \draw[-stealth,color=black] (0,0,0) -- (1,1,0) node [anchor= north] {$\bm{t_3}$};
            \draw[-stealth,color=black] (0,0,0) -- (0.5,0.5,0) node [anchor = west] {$\overline{P}_{f_3}$};
            \end{scope}
            \filldraw [opacity=.33,red] (\ltriangle,0,0) -- (0,\ltriangle,0)
            -- (0,0,\ltriangle) -- cycle;
        \end{tikzpicture}
        \caption{$\overline{P}_{f_3} = [0.5,0.5,0]^{\top}$, $\bm{t_3} = [1,1,0]^{\top}$, and $\theta_3 = 0 \degree$. Since $\overline{P}_{f_3}$ places all of its probability mass uniformly on the target set, the target divergence is minimized at $0 \degree$. \\\\} \label{fig:tardivc}  
    \end{subfigure}
    \caption{These examples visualize the target divergence for various possible combinations of target functions and simplex vectors. Figure \ref{fig:tardivb} demonstrates minimum alignment, while Figure \ref{fig:tardivc} demonstrates maximum alignment.}
    \label{fig:tardiv}
\end{figure}    
        
    \begin{restatable}[Improbability of Favorable Information Resources]{theorem}{iofir}
        Let $\mathcal{D}$ be a distribution over a set of information resources $\mathcal{F}$, let $F$ be a random variable such that $F \sim \mathcal{D}$, let $t \subseteq \Omega$ be an arbitrary fixed $k$-sized target set with corresponding target function $\bm{t}$, and let $q(t,F)$ be the expected per-query probability of success for algorithm $\mathcal{A}$ on search problem $(\Omega,t,F)$. Then, for any $q_{\mathrm{min}} \in [0,1]$,
        \begin{align*}
            \Pr(q(t, F) \geq q_\mathrm{min})  &\leq \frac{p + \bias(\mathcal{D}, \bm{t})}{q_{\mathrm{min}}}
        \end{align*}
        where $p = \frac{k}{|\Omega|}$.
        \label{thm:iofir}
    \end{restatable}
    \noindent
    Since the size of the target set $t$ is usually small relative to the size of the search space $\mathrm{\Omega}$, $p$ is also usually small. Following the above results, we see that the probability that a search problem with an information resource drawn from $\mathcal{D}$ is favorable is bounded by a low value. This bound tightens as we increase our minimum threshold of success, $q_\mathrm{min}$. Notably, our bound relaxes with the introduction of bias. 
    
    \begin{restatable}[Probability of Success Under Bias-Free Search]{corollary}{reducedprob}
    When \(\bias(\mathcal{D}, \bm{t}) = 0\),
    \begin{align*}
        \Pr(q(t, F) \geq q_\mathrm{min}) &\leq \frac{p}{q_{\mathrm{min}}}
    \end{align*}
    \end{restatable}
    \noindent Directly following Theorem~\ref{thm:iofir}, if the algorithm does not induce bias on $\bm{t}$ given a distribution over a set of information resources, the probability of successful search by a favorable information resource cannot be any higher than that of uniform random sampling divided by the minimum performance that we specify. 
    
    \begin{restatable}[Geometric Divergence]{corollary}{geometricdivergence}
    \begin{align*}
              \Pr(q(t, F) \geq q_\mathrm{min}) &\leq \frac{\sqrt{k} \cos (\theta)}{q_{\mathrm{min}}} \\
              &= \frac{\| \bm{t} \| \cos (\theta)}{q_{\mathrm{min}}}
         \end{align*}
    \end{restatable}
    \noindent
     This result shows that greater geometric alignment between the target vector and expected distribution over the search space loosens the upper bound on the probability of successful search. Connecting this to our other results, the geometric alignment can be viewed as another interpretation of the bias the algorithm places on the target set.
    
    \begin{restatable}[Conservation of Bias]{theorem}{conservation}
    Let $\mathcal{D}$ be a distribution over a set of information resources and let $\tau_{k} = \{ \bm{t} | \bm{t} \in \{ 0, 1 \}^{ |\Omega| }, ||\bm{t}|| = \sqrt{k} \}$ be the set of all $|\Omega|$-length $k$-hot vectors. Then for any fixed algorithm $\mathcal{A}$, 
    \begin{align*}
        \sum_{\bm{t} \in \tau_{k}} \bias(\mathcal{D},\bm{t}) = 0
    \end{align*}
    \label{thm:consbias}
    \end{restatable}
    \noindent
    Since bias is a conserved quantity, an algorithm that is biased towards any particular target is equally biased against other targets, as is the case in Schaffer's conservation law for generalization performance~\cite{schaffer1994conservation}. This conservation property holds regardless of the algorithm or the distribution over information resources. Positive dependence between targets and information resources is the grounds for all successful machine learning~\cite{montanez2017dissertation}, and this conservation result is another manifestation of this general property of learning.
    
    \begin{restatable}[Famine of Favorable Information Resources]{theorem}{fofir}
        Let $\mathcal{B}$ be a finite set of information resources and let $t \subseteq \Omega$ be an arbitrary fixed $k$-size target set with corresponding target function $\bm{t}$. Define 
        \begin{align*}
            \mathcal{B}_{q_{\mathrm{min}}} &= \{f \mid f \in \mathcal{B}, q(t,f) \geq q_{\mathrm{min}} \},
        \end{align*}
        where $q(t,f)$ is the expected per-query probability of success for algorithm $\mathcal{A}$ on search problem $(\Omega, t,f)$ and $q_{\mathrm{min}} \in [0,1]$ represents the minimally acceptable per-query probability of success. Then,
        \begin{align*}
            \frac{|\mathcal{B}_{q_{\mathrm{min}}}|}{|\mathcal{B}|} &\leq \frac{p +  \bias(\mathcal{B}, \bm{t})}{q_{\mathrm{min}}}
        \end{align*}
        where $p = \frac{k}{|\Omega|}$.
        \label{thm:fofir}
    \end{restatable}
    \noindent
    This theorem shows us that unless our set of information resources is biased towards our target, only a small proportion of information resources will yield a high probability of search success. In most practical cases, $p$ is small enough that uniform random sampling is not considered a plausible strategy, since we typically have small targets embedded in large search spaces. Thus the bound is typically very constraining. The set of information resources will be overwhelmingly unhelpful unless we restrict the given information resources to be positively biased towards the specified target.
    
    \begin{restatable}[Proportion of Successful Problems Under Bias-Free Search]{corollary}{reducedproportion}
    When \(\bias(\mathcal{B}, \bm{t}) = 0\),
    \begin{align*}
        \frac{|\mathcal{B}_{q_{\mathrm{min}}}|}{|\mathcal{B}|} &\leq \frac{p}{q_{\mathrm{min}}}
    \end{align*}
    \label{cor:reducedproportion}
    \end{restatable}
    \noindent Directly following Theorem \ref{thm:fofir}, if the algorithm does not induce bias on $\bm{t}$ given a set of information resources, the proportion of successful search problems cannot be any higher than the single-query success probability of uniform random sampling divided by the minimum specified performance.
    
    \begin{restatable}[Futility of Bias-Free Search]{theorem}{futility}
    For any fixed algorithm $\mathcal{A}$, fixed target $t \subseteq \Omega$ with corresponding target function $\bm{t}$, and distribution over information resources $\mathcal{D}$, if $\bias(\mathcal{D}, \bm{t}) = 0$, then
    \begin{align*}
        \Pr(\omega \in t; \mathcal{A}) &= p
    \end{align*}
    where $\Pr(\omega \in t; \mathcal{A})$ represents the per-query probability of successfully sampling an element of $t$ using $\mathcal{A}$, marginalized over information resources $F \sim \mathcal{D}$, and $p$ is the single-query probability of success under uniform random sampling.
    \end{restatable}
    \noindent
    This result shows that without bias, an algorithm can perform no better than uniform random sampling. This is a generalization of Mitchell's idea of the futility of removing biases for binary classification \cite{needforbiases} and Monta\~nez's formal proof for the need for bias for multi-class classification \cite{montanez2017dissertation}. This result shows that bias is necessary for any machine learning or search problem to have better than random chance performance.
    
        \begin{restatable}[Famine of Applicable Targets]{theorem}{foat}
    Let $\mathcal{D}$ be a distribution over a finite set of information resources. Define
    \begin{align*}
        \tau_k &= \{t \mid t \subseteq \Omega, |t| = k\} \\
        \tau_{q_{\mathrm{min}}} &= \{t \mid t \in \tau_k, \bias(\mathcal{D}, \bm{t}) \geq q_{\mathrm{min}}\} 
    \end{align*}
    where $\bm{t}$ is the target function corresponding to the target set $t$. Then,
    \[
      \frac{|\tau_{q_\mathrm{min}}|}{|\tau_k|} \leq \frac{p}{p + q_\mathrm{min}} \leq \frac{p}{q_\mathrm{min}}
    \]
    where $p = \frac{k}{|\Omega|}$.
    \end{restatable}
    \noindent
    This theorem shows that the proportion of target sets for which our algorithm is highly biased is small, given that $p$ is small relative to $q_\mathrm{min}$. A high value of $\bias(\mathcal{D}, \bm{t})$ implies that the algorithm, given $\mathcal{D}$, places a large amount of mass on $\bm{t}$ and a small amount of mass on other target functions. Consequently, our algorithm is acceptably biased toward fewer target sets as we increase our minimum threshold of bias.
    
    \begin{restatable}[Famine of Favorable Biasing Distributions]{theorem}{fofbd}
      Given a fixed target function $\bm{t}$, a finite set of information resources $\mathcal{B}$, and a set $\mathcal{P} = \{\mathcal{D}\mid \mathcal{D} \in \mathbb{R}^{|\mathcal{B}|}, \sum_{f \in \mathcal{B}} \mathcal{D}(f) = 1 \}$ of all discrete $|\mathcal{B}|$-dimensional simplex vectors,
      \[
        \frac{\mu(\mathcal{G}_{\bm{t}, q_\mathrm{min}})}{\mu(\mathcal{P})} \leq \frac{p + \bias(\mathcal{B}, \bm{t})}{q_\mathrm{min}}
      \]
      where $\mathcal{G}_{\bm{t}, q_\mathrm{min}} = \{\mathcal{D} \mid \mathcal{D} \in \mathcal{P}, \bias(\mathcal{D}, \bm{t}) \geq q_\mathrm{min}\}$ and $\mu$ is Lebesgue measure.
    \end{restatable}
    \noindent
    We see that the proportion of distributions over $\mathcal{B}$ for which our algorithm is acceptably biased towards a fixed target function $\bm{t}$ decreases as we increase our minimum acceptable level of bias, $q_\mathrm{min}$. Additionally, the greater the amount of bias induced by our algorithm given a set of information resources on a fixed target, the higher the probability of identifying a suitable distribution that achieves successful search. However, unless the set is already filled with favorable elements, finding a minimally favorable distribution over that set is difficult.
    
    \begin{restatable}[Bias Over Distributions]{theorem}{density}
      Given a finite set of information resources $\mathcal{B}$, a fixed target function $\bm{t}$, and a set $\mathcal{P} = \{\mathcal{D}\mid \mathcal{D} \in \mathbb{R}^{|\mathcal{B}|}, \sum_{f \in \mathcal{B}} \mathcal{D}(f) = 1 \}$ of discrete $|\mathcal{B}|$-dimensional simplex vectors,
      \[
        \int_\mathcal{P} \bias(\mathcal {D}, \bm{t}) \dif\mathcal {D} = C \cdot \bias(\mathcal{B}, \bm{t})
      \]
      where $C = \int_\mathcal{P} \dif\mathcal {D}$ is the uniform measure of set $\mathcal{P}$. For an unbiased set $\mathcal{B}$,
      \[
        \int_\mathcal{P} \bias(\mathcal{D}, \bm{t}) \dif\mathcal {D} = 0
      \]
    \end{restatable}
    \noindent
    This theorem states that the total bias on a fixed target function over all possible distributions is proportional to the bias induced by the algorithm given $\mathcal{B}$. When there is no bias over a set of information resources, the total bias over all distributions sums to $0$. It follows that any distribution over $\mathcal{D}$ for which the algorithm places positive bias on $\bm{t}$ is offset by one or more for which the algorithm places negative bias on $\bm{t}$.
    
    \begin{restatable}[Conservation of Bias Over Distributions]{corollary}{conservationdistributions}
    Let $\tau_{k} = \{ \bm{t} | \bm{t} \in \{ 0, 1 \}^{ |\Omega| }, ||\bm{t}|| = \sqrt{k} \}$ be the set of all $|\Omega|$-length $k$-hot vectors. Then, \[\sum_{\bm{t} \in \tau_{k}} \int_{\mathcal{P}} \bias(\mathcal{D}, \bm{t}) \dif\mathcal{D} = 0\]
    \end{restatable}
    \noindent
    This result extends our conservation results, showing that the total bias over all distributions and all $k$-size target sets sums to zero, even when beginning with a set of information resources that is favorably biased towards a particular target.

\section{Examples}

\subsection{Genetic Algorithms}
Genetic algorithms are optimization methods inspired by evolutionary biology~\cite{reeves2002genetic}. We can represent genetic algorithms in our search framework as follows:
\begin{itemize}
    \item $\mathcal{A}$ - a genetic algorithm, with standard variation (mutation, crossover, etc.) operators.
    
    \item $\Omega$ - space of possible configurations (genotypes).
    
    \item $T$ - set of all configurations which perform well on some task.
    
    \item $F$ - a fitness function which can evaluate a configuration's fitness.
    
    \item $(\Omega,T, F)$ - genetic algorithm task.
\end{itemize}

Given any genetic algorithm that is unbiased towards a particular small target when averaged over a set of fitness functions (as in No Free Lunch scenarios), the proportion of highly favorable fitness functions in that set must also be small, which we state as a corollary following directly from Corollary~\ref{cor:reducedproportion}.
\begin{restatable}[Famine of Favorable Fitness Functions]{corollary}{fofff}
      For any fixed target $t \subseteq \Omega$ and fixed genetic algorithm unbiased relative to a finite set of fitness functions $\mathcal{B}$, the proportion of fitness functions in $\mathcal{B}$ with expected per-query probability of success at least $q_{\text{min}}$ is no greater than $|t|/(q_{\text{min}}|\Omega|)$.
    \end{restatable}

\subsection{Binary Classification}
We can cast binary classification as a search problem, as follows~\cite{montanez2017fof}:

\begin{itemize}
    \item $\mathcal{A}$ - classification algorithm, such as an SVM or neural network.
    
    \item $\Omega$ - space of possible binary labelings over an instance space.
    
    \item $t \subseteq \Omega$ - set of all hypotheses with less than 10\% classification error.
    
    \item $F$ - set of training examples, where $F(\emptyset$) is the full set of training data and $F(c)$ is the loss on training data for hypothesis $c$.
    
    \item $(\Omega,t, F)$ - binary classification learning task.
\end{itemize}

In our example, let $|\mathrm{\Omega}| = 2^{100}$. Assume the size of our target set is $|t| = 2^{10}$, the set of training examples $F$ is drawn from a distribution $\mathcal{D}$, and that the minimum performance $q_{\mathrm{min}}$ we want to achieve is $0.5$. Then, by Corollary 1, if our algorithm  (relative to $\mathcal{D}$) does not place any bias on the target set,
\begin{align*}
    \Pr\left(q(t, F) \geq \frac{1}{2}\right) &\leq \frac{p}{q_{\mathrm{min}}}
    = \frac{\frac{2^{10}}{2^{100}}}{\frac{1}{2}}
    = 2^{-89}.
\end{align*}
Thus, the probability that we will have selected a dataset that results in at least our desired level of performance is upper bounded by $2^{-89}$. Notice that if we raised the minimum threshold, then the probability would decrease---favorable datasets would become more unlikely. 

To perform better than uniform random sampling, we would need to introduce bias into the algorithm. For example, predetermined information or assumptions about the target set could be used to determine which hypotheses are more plausible. The principle of Occam's razor \cite{Rasmussen:2000:OR:3008751.3008792} is often used, which is the assumption that the elements in the target set are likely the ``simpler" elements, by some definition of simplicity. Relating this to our formal definition of bias, if we introduce correct assumptions into the algorithm, then the expected alignment of the target set and the induced probability distribution over the search space increases accordingly. 

\section{Conclusion}

We build on the algorithmic search framework and extend Famine of Forte results to search problems with fixed targets and varying information resources. Our notion of bias quantifies the extent to which an algorithm is predisposed to a particular fixed target. We show that bias towards any target necessarily implies bias against the other remaining targets, underscoring the fact that no universally applicable form of bias can exist. Furthermore, one cannot perform better than uniform random sampling without introducing a predisposition in the algorithm towards a desired target---unbiased algorithms are useless. Few information resources can be greatly favorable towards any fixed target, unless the algorithm is already predisposed to the target no matter the information resource given. Thus, in machine learning as elsewhere, biases are needed for better than chance performance. Biases must also be correct, since the effectiveness of any bias depends on how well it aligns with the given target actually being sought.

\bibliographystyle{splncs04}
\bibliography{bibliography}

\section{Appendix: Proofs}
    
    \begin{restatable}[Expected Per Query Performance From Expected Distribution]{lemma}{per-query-performance} 
      This lemma has been proven by Monta\~nez \cite{montanez2017fof} and is directly drawn from \cite{montanez2017fof}. Let \(t\) be a target set, \(q(t, f)\) be the expected per-query probability of success for an algorithm, and $\nu$ be the conditional joint measure induced by that algorithm over finite sequences of probability distributions and search histories, conditioned on external information resource \(f\). Denote a probability distribution sequence by \(\tilde P\) and a search history by h. Let \(\mathcal{U}[\tilde P]\) denote a uniform distribution on elements of \(\tilde P\) and define $\bar{P}(x | f) = \int{\mathbb{E}_{P \sim \mathcal{U}[\tilde{P}]}}[P(x)] \dif\nu (\tilde{P},h|f)$. Then,
      \[
        q(t,f) = \overline{P}(X \in t|f)
      \]
      \label{lem:lemone}
    \end{restatable}
    
    \begin{restatable}[Expectation of Simplex Vectors is Simplex]{lemma}{simplex}
        Let $\mathcal{D}$ be a distribution over a set $\mathcal{F}$ that places $\mathcal{D}(f)$ probability mass on $f \in \mathcal{F}$, and let $\mathcal{X}$ be a set of $|\Omega|$-length simplex vectors, where each $\mathcal{X}_f$ corresponds to a $f \in \mathcal{F}$. Then, $\mathbb{E}_\mathcal{D}[\mathcal{X}_F]$ is a simplex vector. 
        \label{lem:simplex}
    \end{restatable}
    
    \begin{proof}
        By definition of expectation,
        \begin{align*}
            \mathbb{E}_{\mathcal{D}}[\mathcal{X}_F] 
            &= \int_{\mathcal{F}} \mathcal{D}(f) \mathcal{X}_f \dif f
            \end{align*}
            Note that each probability is non-negative and each $\mathcal{X}_f$ is simplex, so the sum has no negative values and thus the expectation has no negative values. To show that the expectation is also a simplex vector, we will sum over its components.
            \begin{align*}
                \int_{\mathcal{F}} \left[\sum_{i = 1}^{|\mathrm{\Omega}|} \mathcal{D}(f) \mathcal{X}_{f_i}\right] \dif f
                &=  \int_{\mathcal{F}} \mathcal{D}(f) \sum_{i = 1}^{|\mathrm{\Omega}|} \mathcal{X}_{f_i} \dif f\\
                &= \int_{\mathcal{F}} \mathcal{D}(f) \dif f\\
                &= 1.
            \end{align*}
        where the penultimate equality follows from the fact that each $\mathcal{X}_f$ is a simplex vector, so must sum to $1$, and the final equality from the fact that $\mathcal{D}$ is a probability distribution on $\mathcal{F}$. Since the expectation is non-negative and the probabilities sum to $1$, each probability mass $\mathcal{D}(f) \in [0,1]$. Thus, the expected value of a set of simplex vectors is a simplex vector.
    \end{proof}
    
    \begin{restatable}[Equivalence of Bias]{lemma}{equivalence} 
      Given a fixed target function $\bm{t}$, a finite set of information resources $\mathcal{B}$, and a set $\mathcal{P} = \{\mathcal{D}| \mathcal{D} \in \mathbb{R}^{|\mathcal{B}|}, \sum_{f \in \mathcal{B}} \mathcal{D}(f) = 1 \}$ of all discrete $|\mathcal{B}|$-dimensional simplex vectors,
      \[
        \mathbb{E}_{\mathcal{U}[\mathcal{P}]}[\bias(D, \bm{t})] = \bias(\mathcal{B}, \bm{t})
      \]
      where $D \sim \mathcal{U}[\mathcal{P}]$.
      \label{lem:equivalencebias}
    \end{restatable}
    \begin{proof}
        Let $F \sim D$. Then,
        \begin{align*}
            \mathbb{E}_{\mathcal{U}[\mathcal{P}]}[\bias(D, \bm{t})]
                &= \mathbb{E}_{\mathcal{U}[\mathcal{P}]}[\mathbb{E}_D[\bm{t}^\top \overline{P}_F] - p] \\
                &= \mathbb{E}_{\mathcal{U}[\mathcal{P}]}\left[\sum_{f \in \mathcal{B}} D(f) \bm{t}^\top \overline{P}_f\right] - p \\
                &= \sum_{f \in \mathcal{B}} \bm{t}^\top \overline{P}_f \mathbb{E}_{\mathcal{U}[\mathcal{P}]}[D(f)] - p
        \end{align*}
        The quantity $\mathbb{E}_{\mathcal{U}[\mathcal{P}]}[D(f)]$ is a uniform expectation on the amount of mass that the random variable $D$ places on resource $f$. Since $\mathcal{P}$ contains all possible distributions over $\mathcal{B}$, under uniform expectation the same amount of probability mass gets placed on each information resource. So, $\mathbb{E}_{\mathcal{U}[\mathcal{P}]}[D(i)] = \mathbb{E}_{\mathcal{U}[\mathcal{P}]}[D(j)]$ for any $i, j \in \mathcal{B}$. Since the probability mass on any two information resources is equivalent and  the total probability mass must sum to one by Lemma~\ref{lem:simplex}, we have  $\mathbb{E}_{\mathcal{U}[\mathcal{P}]}[D(f)] = \frac{1}{|\mathcal{B}|}$. Continuing,
        \begin{align*}
            \mathbb{E}_{\mathcal{U}[\mathcal{P}]}[\bias(D, \bm{t})]
                &= \frac{1}{|\mathcal{B}|} \sum_{f \in \mathcal{B}} \bm{t}^\top \overline{P}_f - p \\
                &= \bias(\mathcal{B}, \bm{t}).
        \end{align*}
    \end{proof}

    \iofir*
    \begin{proof}
        We seek to bound the probability of achieving a successful search on target function $\bm{t}$ with information resource $F$. By Lemma~\ref{lem:lemone}, it follows that 
        \begin{align*}
            \Pr(q(t, F) \geq q_{\mathrm{min}}) &= \Pr(\overline{P}(\omega \in t|F) \geq q_{\mathrm{min}}) \\
                                               &= \Pr(\bm{t}^{\top}\overline{P}_{F} \geq q_{\mathrm{min}})
        \end{align*}
        where $\omega \in t$ means the target function $\bm{t}$ evaluated at $\omega$ is one, and $\overline{P}_F$ represents the $|\mathrm{\Omega}|$-length probability vector defined by $\overline{P} (\cdot|F)$. Applying Markov's Inequality,
        \begin{align*}
            \Pr(q(t, f) \geq q_{\mathrm{min}}) &\leq \frac{1}{q_{\mathrm{min}}} \mathbb{E}_{\mathcal{D}}[\bm{t}^{\top}\overline{P}_{F}] \\\
                                               &= \frac{p + \bias(\mathcal{D}, \bm{t})}{q_{\mathrm{min}}}.
        \end{align*}
    \end{proof}
    
    \reducedprob*
    \begin{proof}
        This result follows directly from Theorem~\ref{thm:iofir}.
    \end{proof}
    
    \conservation*
    \begin{proof} 
        \begin{align*}
        \sum_{\bm{t} \in \tau_{k}} \bias(\mathcal{D},\bm{t}) &= \sum_{\bm{t} \in \tau_{k}} \mathbb{E}_\mathcal{D}[\bm{t}^{\top}P] - p \\
        &= \sum_{\bm{t} \in \tau_{k}} \mathbb{E}_\mathcal{D}[\bm{t}^{\top}P] - \sum_{\bm{t} \in \tau_{k}} \frac{k}{n} \\
        &= \mathbb{E}_\mathcal{D}\Bigg[\sum_{\bm{t} \in \tau_{k}} \bm{t}^{\top}P \Bigg] - \sum_{\bm{t} \in \tau_{k}} \frac{k}{n} \\
        &= \mathbb{E}_\mathcal{D} \Bigg[ \binom{n-1}{k-1} \mathbf{1}^{\top}P \Bigg]-\binom{n}{k}\frac{k}{n} \\
        &= \mathbb{E}_\mathcal{D} \Bigg[ \binom{n-1}{k-1} \Bigg]-\binom{n-1}{k-1} \\ 
        &= 0.
    \end{align*}
    \end{proof}
    
     \geometricdivergence*
    \begin{proof}
        Applying Theorem~\ref{thm:iofir},
        \[
          \Pr(q(t, F) \geq q_{\mathrm{min}}) \leq \frac{p + \bias(\mathcal{D}, \bm{t})}{q_{\mathrm{min}}}
        \]
        By the definition of $\bias(\mathcal{D}, \bm{t})$,
        \begin{align*}
            \Pr(q(t, F) \geq q_{\mathrm{min}}) &\leq \frac{ p + \mathbb{E}_{\mathcal{D}}[\bm{t}^\top \overline{P}_{F}] - p}{q_{\mathrm{min}}} \\
                                               &= \frac{\bm{t}^{\top} \mathbb{E}_{\mathcal{D}} [\overline{P}_{F}]}{q_\mathrm{min}} \\
                                               &= \frac{\|\bm{t}\|}{q_\mathrm{min}} \bigg( \frac{\bm{t}^{\top} \mathbb{E}_{\mathcal{D}}[\overline{P}_{F}]}{\|\bm{t}\|}\bigg)
        \end{align*}
        By Lemma~\ref{lem:simplex}, $\mathbb{E}_{\mathcal{D}}[{\overline{P}_{F}}]$ is a simplex vector, so its terms sum to $1$. Thus, $\|\mathbb{E}_{\mathcal{D}}[{\overline{P}_{F}}]\| \leq 1$. So,
        \begin{align*}
            \Pr(q(t, F) \geq q_{\mathrm{min}}) &\leq \frac{\|\bm{t}\|}{q_\mathrm{min}} \bigg( \frac{\bm{t}^{\top} \mathbb{E}_{\mathcal{D}}[\overline{P}_{F}]}{\|\bm{t}\| \|\mathbb{E}_{\mathcal{D}}[{\overline{P}_{F}}]\|}\bigg) \\
                                               &= \frac{\|\bm{t}\|}{q_\mathrm{min}} \cos \bigg( \arccos \bigg( \frac{\bm{t}^{\top} \mathbb{E}_{\mathcal{D}}[\overline{P}_{F}]}{\|\bm{t}\| \|\mathbb{E}_{\mathcal{D}}[{\overline{P}_{F}}]\|}\bigg) \bigg)
        \end{align*}
        By the definition of target divergence, we have
        \[
          \Pr(q(t, F) \geq q_{\mathrm{min}}) \leq \frac{\|\bm{t}\| \cos(\theta)}{q_\mathrm{min}}.
        \]
    \end{proof}
    
    \fofir*
    \begin{proof}
        We seek to bound the proportion of successful search problems for which $q(t, f) \geq q_{\mathrm{min}}$ for any threshold $q_{\mathrm{min}} \in (0, 1]$. Let $F \sim \mathcal{U}[\mathcal{B}]$. Then, 
        \begin{align*}
            \frac{|\mathcal{B}_{q_{\mathrm{min}}}|}{|\mathcal{B}|} &= \frac{1}{ |\mathcal{B}|} \sum_{f \in \mathcal{B}} \mathbbm{1}_{q(t,f) \geq q_{\mathrm{min}}}\\
                                               &=  \mathbb{E}_{\mathcal{U}[\mathcal{B}]}[\mathbbm{1}_{q(t,F) \geq q_{\mathrm{min}}}] \\
                                               &= \Pr(q(t, F) \geq q_{\mathrm{min}}).
        \end{align*}
        Let $\omega \in t$ mean the target function $\bm{t}$ evaluated at $\omega$ is one. Then, by applying Lemma~\ref{lem:lemone},
        \begin{align*}
            \frac{|\mathcal{B}_{q_{\mathrm{min}}}|}{|\mathcal{B}|} &= \Pr(\overline{P}(\omega \in t|F) \geq q_{\mathrm{min}}) \\
                                               &= \Pr(\bm{t}^{\top} \overline{P}_{F} \geq q_\mathrm{min}).
        \end{align*}
        Applying Markov's Inequality and by the definition of $\bias(\mathcal{B}, \bm{t})$,
        \begin{align*}
            \frac{|\mathcal{B}_{q_{\mathrm{min}}}|}{|\mathcal{B}|} &\leq \frac{\mathbb{E}_{\mathcal{U}[\mathcal{B}]} [\bm{t}^{\top} \overline{P}_{F}]}{q_{\mathrm{min}}} \\
                                               &= \frac{p + \bias(\mathcal{B},\bm{t})}{q_{\mathrm{min}}}.
        \end{align*}
    \end{proof}
    
    \reducedproportion*
    \begin{proof}
        This result follows directly from Theorem~\ref{thm:fofir}.
    \end{proof}
    
    \futility*
    \begin{proof}
        Let $\mathcal{F}$ be the space of possible information resources. Then,
        \begin{align*}
            \Pr(\omega \in t; \mathcal{A}) 
                &= \int_\mathcal{F} \Pr(\omega \in t, f; \mathcal{A}) \dif f\\
                &= \int_\mathcal{F} \Pr(\omega \in t \mid f; \mathcal{A})\Pr(f) \dif f.
        \end{align*}
        Since we are considering the per-query probability of success for algorithm $\mathcal{A}$ on $t$ using information resource $f$, we have
        \[
          \Pr(\omega \in t \mid f; \mathcal{A}) = \overline{P}(\omega \in t \mid f).
        \]
        Also note that $\Pr(f) = \mathcal{D}(f)$ by the fact that $F \sim \mathcal{D}$. Making these substitutions, we obtain
        \begin{align*}
            \Pr(\omega \in t; \mathcal{A}) 
                &= \int_\mathcal{F} \overline{P}(\omega \in t \mid f)\mathcal{D}(f) \dif f\\
                &= \mathbb{E}_{\mathcal{D}}\left[\overline{P}(\omega \in t \mid F)\right]\\
                &= \mathbb{E}_{\mathcal{D}}\left[\mathbf{t}^{\top}\overline{P}_F\right]\\
                &= \bias(\mathcal{D}, \bm{t}) + p\\
                &= p.
        \end{align*}
    \end{proof}
    
    \foat*
    \begin{proof}
        First, note that the size of $\tau_k$ is equivalent to the number of $k$-sized subsets of a $|\mathrm{\Omega}|$-size set, $\binom{|\mathrm{\Omega}|}{k}$. The size of $\tau_{q_\mathrm{min}}$ is the number of target sets in $\tau_k$ for which $\bias(\mathcal{D}, \bm{t}) \geq q_\mathrm{min}$. Let $F \sim \mathcal{D}$ and $T \sim \mathcal{U}[\tau_k]$. Then,
        \begin{align*}
            |\tau_{q_\mathrm{min}}| 
                &= \sum_{\bm{t} \in \tau_k} \mathbbm{1}_{\bias{(\mathcal{D}, \bm{t}) \geq q_\mathrm{min}}} \\
                &= \binom{|\mathrm{\Omega}|}{k} \sum_{t \in \tau_k} \binom{|\mathrm{\Omega}|}{k}^{-1} \mathbbm{1}_{\bias{(\mathcal{D}, \bm{t}) \geq q_\mathrm{min}}} \\
                &= \binom{|\mathrm{\Omega}|}{k} \mathbb{E}_{\mathcal{U}[\tau_k]}[\mathbbm{1}_{\bias(\mathcal{D}, T) \geq q_\mathrm{min}}] \\
                &= \binom{|\mathrm{\Omega}|}{k} \Pr(\bias(\mathcal{D}, T) \geq q_\mathrm{min}) \\
                &= \binom{|\mathrm{\Omega}|}{k} \Pr(p + \bias(\mathcal{D}, T) \geq p + q_\mathrm{min}) \\
                &= \binom{|\mathrm{\Omega}|}{k} \Pr(\mathbb{E}_\mathcal{D}[T^\top \overline{P}_F] \geq p + q_\mathrm{min}).
        \end{align*}
        Applying Markov's Inequality,
        \begin{align*}
            |\tau_{q_\mathrm{min}}|
                &\leq \frac{\binom{|\mathrm{\Omega}|}{k} \mathbb{E}_{\mathcal{U}[\tau_k]}[\mathbb{E}_\mathcal{D}[T^\top \overline{P}_F]]}{p + q_\mathrm{min}} \\
                &= \frac{\binom{|\mathrm{\Omega}|}{k} \sum_{\bm{t} \in \tau_k}\binom{|\mathrm{\Omega}|}{k}^{-1} \mathbb{E}_\mathcal{D}[\bm{t}^\top \overline{P}_F]}{p + q_\mathrm{min}} \\
                &= \frac{\mathbb{E}_\mathcal{D}[\overline{P}_F^\top \sum_{\bm{t} \in \tau_k} \bm{t}]}{p + q_\mathrm{min}} \\
                &= \frac{\mathbb{E}_\mathcal{D}[\overline{P}_F^\top \mathbf{1} \binom{|\mathrm{\Omega}|-1}{k-1}]}{p + q_\mathrm{min}} \\
                &= \frac{\binom{|\mathrm{\Omega}|-1}{k-1}\mathbb{E}_\mathcal{D}[\overline{P}_F^\top \mathbf{1}]}{p + q_\mathrm{min}} \\
                &= \frac{\binom{|\mathrm{\Omega}|-1}{k-1}}{p + q_\mathrm{min}}.
        \end{align*}
        Thus,
        \begin{align*}
            \frac{|\tau_{q_\mathrm{min}}|}{|\tau_k|} 
                &\leq \frac{\binom{|\mathrm{\Omega}| - 1}{k-1}}{\binom{|\mathrm{\Omega}|}{k}(p + q_\mathrm{min})} \\
                &= \frac{\binom{|\mathrm{\Omega}| - 1}{k-1}}{\frac{|\mathrm{\Omega}|}{k}\binom{|\mathrm{\Omega}|-1}{k-1}(p + q_\mathrm{min})} \\
                &= \frac{p}{p+q_{\mathrm{min}}} \\
                &\leq \frac{p}{q_\mathrm{min}}.
        \end{align*}
    \end{proof}
    
    \fofbd*
    \begin{proof}
        Let $D \sim \mathcal{U}[\mathcal{P}]$. Then,
        \begin{align*}
            \frac{\mu(\mathcal{G}_{t, q_{\mathrm{min}}})}{\mu(\mathcal{P})}
            &= \Pr(\bias(D, \bm{t}) \geq q_\mathrm{min}) \\
            &= \Pr(p + \bias(D, \bm{t}) \geq p + q_\mathrm{min}) \\
            &= \Pr( \mathbb{E}_D[\bm{t}^\top \overline{P}_F] \geq p + q_\mathrm{min}).
        \end{align*}
        Applying Markov's inequality and Lemma~\ref{lem:equivalencebias},
        \begin{align*}
            \frac{\mu(\mathcal{G}_{\bm{t}, q_{\mathrm{min}}})}{\mu(\mathcal{P})}
            &\leq \frac{\mathbb{E}_{\mathcal{U}[\mathcal{P}]}[\mathbb{E}_D[\bm{t}^\top \overline{P}_F]]}{p + q_\mathrm{min}} \\
            &= \frac{p + \mathbb{E}_{\mathcal{U}[\mathcal{P}]}[\bias(D, \bm{t})]}{p + q_{\mathrm{min}}} \\
            &= \frac{p + \bias(\mathcal{B}, \bm{t})}{p + q_\mathrm{min}} \\
            &\leq \frac{p + \bias(\mathcal{B}, \bm{t})}{q_\mathrm{min}}.
        \end{align*}
    \end{proof}
    
    \density*
    \begin{proof}
        \begin{align*}
            \int_{\mathcal{P}} \bias(\mathcal{D}, \bm{t}) \dif\mathcal{D}
                &= C\int_{\mathcal{P}} \frac{1}{C} \bias(\mathcal{D}, \bm{t}) \dif\mathcal{D} \\
                &= C \cdot \mathbbm{E}_{\mathcal{U}[\mathcal{P}]} [\bias(D, \bm{t})] 
        \end{align*}
        By Lemma~\ref{lem:equivalencebias},
        \begin{align*}
            \int_{\mathcal{P}} \bias(\mathcal{D}, \bm{t}) \dif\mathcal{D}
                &= C \cdot \bias(\mathcal{B}, \bm{t}).
        \end{align*}
    \end{proof}
    
    \conservationdistributions*
    \begin{proof}
        By Theorem~\ref{thm:consbias},
        \begin{align*}
            \sum_{\bm{t} \in \tau_{k}} \int_{\mathcal{P}} \bias(\mathcal{D}, \bm{t}) \dif\mathcal{D} &= \int_{\mathcal{P}} \bigg( \sum_{t \in \tau_{k}} \bias(\mathcal{D}, \bm{t}) \bigg) \dif\mathcal{D} \\
            &= \int_{\mathcal{P}} 0 \dif\mathcal{D} \\
            &= 0.
        \end{align*}
    \end{proof}

\end{document}